\documentclass{article}
\usepackage{graphicx} 
\usepackage[utf8]{inputenc}
\usepackage{amsmath}
\usepackage{amssymb}
\usepackage{mathtools}
\usepackage{amsthm}
\usepackage{centernot}
\usepackage{algorithm}
\usepackage{algorithmic}
\usepackage{todonotes}

\usepackage[numbers]{natbib}
\usepackage[pdftex,colorlinks,pagebackref=true]{hyperref}
\usepackage[margin=1.0in]{geometry}
\usepackage{cjhebrew}

\newtheorem{theorem}{Theorem}[section]
\newtheorem{lemma}[theorem]{Lemma}

\newtheorem{counter-example}[theorem]{Counter example}

\newtheorem{open question}[theorem]{Open question}
\newtheorem{corollary}[theorem]{Corollary}

\newcommand{\inner}[1]{{\left\langle #1 \right\rangle}}

\newcommand{\var}{\mathrm{Var}}

\newcommand{\E}{\mathbb{E}}
\newcommand{\reals}{\mathbb{R}}
\newcommand{\sphere}{\mathbb{S}}

\newcommand{\so}{\mathbb{SO}}

\newcommand{\sign}{\mathrm{sign}}

\newcommand{\spec}{\mathrm{sp}}
\newcommand{\diag}{\mathrm{diag}}

\newcommand{\bx}{\mathbf{x}}

\newcommand{\by}{\mathbf{y}}
\newcommand{\bw}{\mathbf{w}}

\newcommand{\bg}{\mathbf{g}}

\newcommand{\vbx}{\vec{\mathbf{x}}}
\newcommand{\vby}{\vec{\mathbf{y}}}

\newcommand{\cc}{\mathcal{C}}

\newif\ifdraft
\drafttrue  
\draftfalse

\title{Existence of Adversarial Examples for Random Convolutional Networks via  Isoperimetric Inequalities on $\so(d)$}

\author{Amit Daniely}
\begin{document}

\maketitle

\begin{abstract}
   We show that adversarial examples exist for various random convolutional networks, and furthermore, that this is a relatively simple consequence of the isoperimetric inequality on the special orthogonal group $\so(d)$. This extends and simplifies a recent line of work \citep{daniely2020most, bubeck2021single, bartlett2021adversarial, montanari2023adversarial} which shows similar results for random fully connected networks.
\end{abstract}

\section{Introduction}
Adversarial examples, first observed by~\citet{szegedy2014intriguing},  were studied extensively in recent years, with several attacks (e.g. ~\cite{athalye2018obfuscated, carlini2017adversarial, carlini2018audio, goodfellow2014explaining, grosse2017statistical}) and defense methods (e.g. \cite{papernot2016distillation, papernot2017practical, madry2017towards, wong2018provable, feinman2017detecting}) being developed, as well as various attempts to explain why they exists (e.g. \cite{fawzi2018adversarial, shafahi2018adversarial, shamir2019simple, schmidt2018adversarially, bubeck2019adversarial}).
In particular, one line of work aims at explaining the phenomenon of adversarial examples by proving that they exists and can be found in random networks. \citet{daniely2020most} show that adversarial examples exists in random constant depth fully connected ReLU networks in which each layer reduces the width. Moreover, they showed that gradient flow, as well as gradient decent with sufficiently small step size, will find these adversarial examples. \citet{bartlett2021adversarial} following \citet{bubeck2021single} improved on \citet{daniely2020most} as they replaced the assumption that the dimension decreases with a very mild assumption that there is no exponential gap between the width of different layers. They also showed that gradient decent with a single step will find an adversarial example. \citet{montanari2023adversarial} further improved these results, as they completely dropped the width requirement.

We continue this line of work. Our contribution is twofold. First, we extend the family of architectures for which these results are applicable. We show the existence of adversarial examples in random {\em convolutional} ReLU networks of any constant depth, with no restriction on the width. For odd activations (such as sigmoids like $\sigma(x) = \tan^{-1}(x)$ and $\sigma(x) = \frac{e^x-e^{-x}}{e^x+e^{-x}}$) we show that adversarial examples exists for a broader class of architectures: We show that adversarial examples exists provided that the first layer is convolutional. This result is valid for any layered architecture in the remaining layers, with no restriction on the width nor on the depth. Our second contribution is that we substantially simplify the proofs of such results, and show that the existence of adversarial examples is a relatively simple consequence of the isoperimetric inequality on the special orthogonal group $\so(d)$. On the flip side as opposed to previous papers in this line of research, our techniques are less constructive, and we do not present an algorithm for finding an adversarial perturbation.

\subsection{Related Work}

Several recent theoretical studies have explored the fundamental reasons behind the existence of adversarial examples in machine learning. \citet{schmidt2018adversarially} demonstrate that training adversarially robust classifiers can require a significantly larger sample complexity compared to standard training, while \citet{bubeck2019adversarial} highlight scenarios where adversarially robust training is computationally more demanding.

\citet{fawzi2018adversarial, mahloujifar2019curse} leverage concentration of measure results to show that for various subsets of $\reals^d$,
such as the sphere, ball, or cube, any partition of the space into a small number of subsets with non-negligible measure (with respect to the uniform distribution) will necessarily lead to an abundance of adversarial examples. In other words, most points will have a nearby example that belongs to a different subset of the partition, implying that \emph{any} classifier implementing such a partition will be susceptible to adversarial examples. \citet{shafahi2018adversarial} extend these findings to classification tasks where examples are generated by specific generative models.

Further, \citet{vardi2022gradient} and \citet{melamed2023adversarial} establish the existence of adversarial examples in trained depth-two neural networks. Lastly, \citet{shamir2019simple} investigate adversarial vulnerability with respect to the $\ell^0$ norm.

\section{Setting}

\subsubsection*{Notation}
We will use $M_{d\times n}$ to denote the space if $d\times n$ matrices.
We will use the standard Euclidean/Frobenius norm on the spaces $\reals^d$, $\left(\reals^d\right)^n$ and $M_{d\times n}$. We will denote the spectral norm of $A\in M_{d\times n}$ by $\|A\|_\spec$. Similarly, for a sequence $\vbx=(\bx_1,\ldots,\bx_n)\in \left(\reals^d\right)^n$ of $n$ vectors in $\reals^d$
we will denote by $\|\vbx\|_\spec$ the spectral norm of the $d\times n$ matrix whose $i$'th column is $\bx_i$. For a function $\sigma:\reals\to\reals$ and $\bx\in\reals^d$ we  denote by $\sigma(\bx)$ the vector $(\sigma(x_1),\ldots,\sigma(x_d))$. We will use $\gtrsim$ for denoting inequality up to a multiplicative constant.

\subsubsection*{Convolutional Networks}
A {\em layer} is a function $F:(\reals^{d_1})^n\to\reals^{d_2}$ of the form $F(\bx) = \sigma(W\bx)$ for a $d_2\times (d_1n)$ matrix $W$ and $\sigma:\reals\to\reals$ (that is called the {\em activation function} of the layer). A layered network is a composition of several layers.
A layer is {\em convolutional of width $w$, stride $s$ and $d_1$ channels}, for $w\le n$ such that $s$ divides $n-w$, if it is of the form $F_W(\vbx) = (\sigma(WT_{0,s,w}(\vbx)),\ldots,\sigma(WT_{\frac{n-w}{s},s,w}(\vbx)))$ where $T_{i,s,w}(\vbx) = (\bx^{is+1},\ldots,\bx^{is+w})$ and $W$ is a  $d_2\times (d_1w)$ matrix. We note that for the sake of simplicity, we consider one-dimensional convolutions, despite that our results can be phrased for general convolutional layers. We also note that the definition of convolutional layer encompasses fully connected layers as well, by taking the width $w$ to be $n$.

\subsubsection*{Random Convolutional Networks}
A {\em random convolutional layer} is a random function $F_W$ where $W$ is a random $d_2\times (d_1w)$ matrix.  We say that $W$  is {\em regular}, if for any orthogonal $U\in M_{(d_1w),(d_1w)}$ the distribution of $W$ and $WU$ are identical. 
We say that $W$ is {\em Xavier} \cite{glorot2010understanding} if its entries are i.i.d. centered Gaussians with variance $\frac{1}{d_1w}$ (note that a Xavier matrix is necessarily regular).
We say that $F_W$ is {\em regular/Xavier} if $W$ is a regular/Xavier random matrix.
We note that it is common to initialize neural networks with
regular random matrices. For instance Xavier matrices and random orthogonal matrices are standard choices for initial weights.

\section{Main Results}

Our first result assumes that the activation functions is odd (that is, satisfy $\sigma(-x) = -\sigma(x)$). Examples to such activations are sigmoids like $\sigma(x) = \tan^{-1}(x)$ and $\sigma(x) = \frac{e^x-e^{-x}}{e^x+e^{-x}}$. Under this assumption, we can show that adversarial examples exist in a quite general setting. That is, we show that adversarial examples exist if the random network is a regular random convolutional layer, followed by {\em any} layered network (with no restriction on the width, depth, etc.).

\begin{theorem}\label{thm:odd}
    Let $f:\left(\reals^d\right)^n\to\reals$ be a random layered network whose first layer is a regular random convolutional layer that is independent from the remaining layers. Assume that the activations in all layers are odd and that $d$ is even. Fix $\vbx_0\in \left(\reals^d\right)^n$. Then, w.p. $1-2e^{-\frac{\tau^2}{32}}$ over the choice of $f$ either $f(\vbx_0)=0$ or there is $\vbx$ such that $\|\vbx-\vbx_0\| \le \left(\frac{\tau }{\sqrt{d-2}}\cdot\frac{\|\vbx_0\|_\spec}{\|\vbx_0\|}\right)\cdot \|\vbx_0\|$ and $\sign(f(\vbx))\ne \sign(f(\vbx_0))$
\end{theorem}
Two remarks are in order. The first is that for any $\vbx_0\in \left(\reals^d\right)^n$ we have $\frac{1}{\sqrt{\min(d,n)}}\le \frac{\|\bx_0\|_\spec}{\|\bx_0\|}\le 1$, and that for ``typical" (e.g. random) $\vbx_0$ we have $ \frac{\|\bx_0\|_\spec}{\|\bx_0\|}\approx   \frac{1}{\sqrt{\min(d,n)}}$. Thus, Theorem \ref{thm:odd} guarantees that for any $\vbx_0$, w.h.p. there is an adversarial example $\vbx$ with $\|\vbx-\vbx_0\| \le\frac{1}{\sqrt{d}}\|\vbx_0\|$, which essentially matches state of the art~\citep{daniely2020most, bartlett2021adversarial, bubeck2021single, montanari2023adversarial} for fully connected networks ($n=1$). For a ``typical" $\vbx_0$ the guarantee is stronger and implies that there is an adversarial example $\vbx$ with $\|\vbx-\vbx_0\| \le\frac{1}{\sqrt{\min(d^2,dn)}}\|\vbx_0\|$. The second remark is about the possibility that $f(\vbx_0)=0$. We note that for a ``reasonable" model of random  networks, this probability is $o_d(1)$ or even $0$. For such models, the conclusion of Theorem \ref{thm:odd} implies that w.p. $(1-o_d(1)\left(1-2e^{-\frac{\tau^2}{32}}\right)$ over the choice of $f$ there is $\vbx$ such that $\|\vbx-\vbx_0\| \le \left(\frac{\tau }{\sqrt{d-2}}\cdot\frac{\|\vbx_0\|_\spec}{\|\vbx_0\|}\right)\cdot \|\vbx_0\|$ and $\sign(f(\vbx))\ne \sign(f(\vbx_0))$.

Our second result considers constant depth random convolutional networks with ReLU activation.

\begin{theorem}\label{thm:relu_nice}
    Let $f:(\reals^d)^n\to\reals$ be a random layered network with $l$ independent convolutional Xavier layers with ReLU activation, except the last layer which has linear activation.
    Assume that the number of channels in each layer, as well as $d$ are all $\omega(\log(nd))$. Fix $\vbx_0\in \left(R\cdot\sphere^{d-1}\right)^n$. Then, w.p. $(1-o_d(1))\left(1-2e^{-\Omega(\tau^2/\log^2(d))}\right)$ over the choice of $f$, either $f(\bx_0)=0$ or there is $\vbx$ such that $\|\vbx-\vbx_0\| \le \left(\frac{\tau }{\sqrt{d-2}}\cdot\frac{\|\vbx_0\|_\spec}{\|\vbx_0\|}\right)\cdot \|\vbx_0\|$ and $\sign(f(\vbx))\ne \sign(f(\vbx_0))$. The asymptotic notations depend only on the depth $l$.
\end{theorem}

\section{Isoperimetric inequalities for $\so(d)$ and $\so(d)$-metric-spaces}
Let $\so(d)$ the special orthogonal group. That is, $\so(d)$ is the group of orthogonal matrices with determinant $1$. 
We will view $\so(d)$ as a metric space w.r.t. the metric induced by the Frobenius norm. We note that this metric is $\so(d)$-invariant in the sense that $d(UV_1,UV_2) = d(V_1,V_2)$ for any $U,V_1,V_2\in\so(d)$.
Let $\mu$ be the uniform (Haar) probability measure on $\so(d)$. We note that $\mu$ is the unique probability measure on $\so(d)$ that is $\so(d)$-invariant. That is, it is the unique probability measure on $\so(d)$ that satisfies $\mu(UA) = \mu(A) = \mu(AU)$ for any measurable $A\subseteq\so(d)$ and any $U\in \so(d)$.

The results in this paper will make a crucial use of isoperimetric inequalities on $\so(d)$, which we introduce next. To this end, for $A\subseteq \so(d)$, $U\in \so(d)$ and $\epsilon>0$ we let
\[
d(A,U) = \min_{V\in A}\|V-U\|
\]
and
\[
A_\epsilon = \left\{ U\in \so(d) : d(A,U)\le \epsilon \right\}
\]
Isoperimetric inequalities shows that for relatively small $\epsilon>0$, $\mu(A_\epsilon)\gg\mu(A)$. These inequalities are based on the concentration of measure property on $\so(d)$.

\begin{theorem}[Measure concentration for $\so(d)$. E.g. \cite{meckes2019random} section 5.3]
    Let $f:\so(d)\to\reals$ be $L$-Lipschitz w.r.t. the Frobenius metric. We have $\mu\left(f \ge \E_\mu f + \epsilon\right)\le e^{-\frac{(d-2)\epsilon^2}{8L^2}}$
\end{theorem}

\begin{corollary}[Isoperimetric inequality for $\so(d)$]
    For $A\subseteq \so(d)$ we have $\mu(A_\epsilon) \ge 1 - e^{-\frac{(d-2)\epsilon^2(\mu(A))^2}{8}}$
\end{corollary}
\begin{proof}
    Define $f(x) = \max(0,\epsilon - d(A,x))$. We note that $f$ is $1$-Lipschitz. Also, $f(x)\ge \epsilon\cdot1_A(x)$. Hence, we have $\E_\mu f\ge\epsilon\E_\mu 1_A=\mu(A)\epsilon$. This implies that
    \[
    \mu\left(A_\epsilon^\complement\right) = \mu(f = 0) \le \mu\left(f \le \E_\mu f -\mu(A)\epsilon\right)\le e^{-\frac{(d-2)\epsilon^2(\mu(A))^2}{8}}
    \]
\end{proof}

We next extend the $\so(d)$-isoperimetric inequalities to metric spaces on which $\so(d)$ act. Let $X$ be a metric space. We say that $\so(d)$ {\em acts on $X$} if there is a mapping $(U,\bx)\mapsto U\bx$ such that  for any $U,V\in \so(d)$ and $\bx,\by\in X$ we have (i) $V(U\bx) = (VU)\bx$, (ii) $I\bx=\bx$, and (iii) $d(U\bx,U\by) = d(\bx,\by)$. We will refer to a metric space on which $\so(d)$ acts as an {\em $\so(d)$-metric-space}.

Examples of $\so(d)$-metric-space are $\reals^d$ an $\sphere^{d-1}$ (which are the input spaces for fully connected networks). Here, the action of $\so(d)$ is standard matrix multiplication. Additional examples are $\left(\reals^d\right)^n$ and $\left(\sphere^{d-1}\right)^n$ (which are the input spaces for convolutional networks). Here, the action is
\[
U\cdot(\bx_1,\ldots,\bx_n):= (U\bx_1,\ldots,U\bx_n)
\]
We say that an $\so(d)$-metric-space is {\em transitive} if for any $\bx,\by\in X$ there is $U\in \so(d)$ such that $\by = U\bx$. It is clear $\sphere^{d-1}$ is transitive and that for any $\so(d)$-metric-space $X$ and any $\bx\in X$, the {\em orbit} of $\bx$, i.e. $\cc(\bx) = \{U\bx : U\in\so(d)\}$, is transitive. On the other hand, $\reals^d$ as well as $\left(\reals^d\right)^n$ and $\left(\sphere^{d-1}\right)^n$ are not transitive. We define the Haar probability measure on a transitive $\so(d)$-metric-space $X$ as follows. Choose some $\bx\in X$. For any $A\subseteq X$ we let 
\[
\mu_X(A) = \mu(\{U\in\so(d) : U\bx\in A\})
\]
Note that $\mu_X$ in the l.h.s. denotes the Haar measure on $X$ while in the r.h.s. $\mu$ denotes the Haar measure on $\so(d)$.  We note that the definition of $\mu_X$ does not depend on the choice of $\bx$. Indeed, for any $\by\in X$, since $X$ is transitive we have $\by = V\bx$ for some $V\in\so(d)$. Hence,
\begin{eqnarray*}
\mu(\{U\in\so(d) : U\by\in A\}) &=& \mu(\{U\in\so(d) : UV\bx\in A\})
\\
&=& \mu(\{U\in\so(d) : U\bx\in A\}\cdot V^{-1})    
\\
&\stackrel{\mu\text{ is $\so(d)$-invariant}}{=}& \mu(\{U\in\so(d) : U\bx\in A\})    
\end{eqnarray*}
We also note that $\mu_X$ is $\so(d)$-invariant. Indeed, for $A\subseteq X$, $V\in \so(d)$ and $\bx\in X$ we have
\begin{eqnarray*}
\mu_X(VA) &=& \mu(\{U\in\so(d) : U\bx\in VA\})
\\
&=& \mu(\{U\in\so(d) : V^{-1}U\bx\in A\})
\\
&=& \mu(V\cdot\{U\in\so(d) : U\bx\in A\}\cdot )    
\\
&\stackrel{\mu\text{ is $\so(d)$-invariant}}{=}& \mu(\{U\in\so(d) : U\bx\in A\})    
\\
&=& \mu_X(A)
\end{eqnarray*}
And similarly $\mu_X(AV)=\mu_X(A)$.

We say that the action of $\so(d)$ on an $\so(d)$-metric-space $X$ is {\em $L$-Lipschitz} (or that $X$ is $L$-Lipschitz) if for any $\bx$ the mapping $U\mapsto U\bx$ is $L$-Lipschitz. 
We note that in the case that $X$ is transitive, the action is $L$-Lipschitz if the mapping $U\mapsto U\bx_0$ is $L$-Lipschitz for some $\bx_0\in X$. Indeed, in this case, given any $\bx\in X$, there is $V\in\so(d)$ such that $\bx = V\bx_0$. Now, the mapping $U\mapsto U\bx$ is $L$-Lipschitz as the  composition of the $1$-Lipschitz mapping $U\mapsto UV$ followed by the $1$-Lipschitz mapping $U\mapsto U\bx_0$. We will use the following Lemma
\begin{lemma}\label{lem:conv_inp_space_is_lip}
    Let $\vbx = (\bx_1,\ldots,\bx_n)\in \left(\reals^d\right)^n$. We have that the action of $\so(d)$ on $\cc(\vbx)$ is $\|\vbx\|_\spec$-Lipschitz.
\end{lemma}
\begin{proof}
    Let $X$ be $d\times n$ the matrix whose $i$'th column in $\bx_i$. Since $\cc(\vbx)$ is transitive it is enough to show that $U\mapsto U\vbx$ is $\|\vbx\|_\spec$-Lipschitz. Indeed, for $U,V\in\so(d)$ we have
    \begin{eqnarray*}
        \| U\vbx - V\vbx\| &=&   \| ((U-V)\bx_1,\ldots,(U-V)\bx_n)\|
        \\
        & = & \| (U-V)X \|
        \\
        &\le & \| U-V \|\cdot \| X \|_\spec
    \end{eqnarray*}
\end{proof}
We will use the following generalization of the isoperimetric inequality on $\so(d)$. To this end, for $A\subseteq X$, $\bx\in X$ and $\epsilon>0$ we let $d(A,\bx) = \min_{\by\in A}d(\by,\bx)$ and $A_\epsilon = \left\{ \bx\in X : d(A,\bx)\le \epsilon \right\}$. 
\begin{theorem}[Isoperimetric inequality for $\so(d)$-metric-spaces]\label{thm:iso_for_inv}
Let $X$ be transitive $\so(d)$-metric-space. Assume that the action of $\so(d)$ is $L$-Lipschitz. Then, for any $A\subseteq X$ we have $\mu_X(A_\epsilon) \ge 1 - e^{-\frac{(d-2)\epsilon^2(\mu_X(A))^2}{8L^2}}$    
\end{theorem}
\begin{proof}
Choose some $\bx_0\in X$ and let $f:\so(d)\to X$ be the function $f(U) = U\bx_0$. We note that $f$ is $L$-Lipschitz and that $\mu(f^{-1}(C))=\mu_X(C)$ for any $C\subseteq X$. Now we have
\begin{eqnarray*}
    \mu_X(A_\epsilon) &=& \mu(f^{-1}(A_\epsilon))
    \\
    &\stackrel{(f^{-1}(A))_{\epsilon/L}\subseteq f^{-1}(A_\epsilon)}{\ge} & \mu((f^{-1}(A))_{\epsilon/L})
    \\
    &\stackrel{\so(d)\text{-isoperimetric inequality}}{\ge} & 1 - e^{-\frac{(d-2)\epsilon^2(\mu(f^{-1}(A)))^2}{8L^2}}
    \\
    &= & 1 - e^{-\frac{(d-2)\epsilon^2(\mu_X(A))^2}{8L^2}}
\end{eqnarray*}
\end{proof}

\section{Proof of the Main Results}
Let $X$ be an $L$-Lipschitz $\so(d)$-metric-space. For $f:X\to Y$ and $U\in\so(d)$ we define the function $Uf:X\to Y$ by $(Uf)(\bx)=f(U^{-1}\bx)$.
Let $f$ be a random function from $X$ to $\reals$. We say that $f$ is {\em $\so(d)$-invariant} if for any $U\in\so(d)$ the distribution of $Uf$ and $f$ are identical. An example for an $\so(d)$-invariant random function is a random convolutional or fully connected network.
\begin{lemma}\label{lem:conv_nets_are_inv}
    Let $f:\left(\reals^d\right)^n\to\reals$ be a random layered network whose first layer is a regular random convolutional layer that is independent from the remaining layers. Then, $f$ is $\so(d)$-invariant.
\end{lemma}
\begin{proof}
    We have $f = g\circ F_W$ where $F_W(\vbx) = (\sigma(WT_{1,s,w}(\vbx)),\ldots,\sigma(WT_{\frac{n-w}{s},s,w}(\vbx)))$ is a convolutional layer for a regular random matrix $W$ that is independent of $g$. For $U\in\so(d)$ We have
    \begin{eqnarray*}
        Uf(\vbx) &=& g\circ F_W(U^{-1}\vbx) 
        \\
        &=& g\left( \sigma(WT_{1,s,w}(U^{-1}\vbx)),\ldots,\sigma(WT_{\frac{n-w}{s},s,w}(U^{-1}\vbx)) \right)
        \\
        &=& g\left( \sigma(WU^{-1}T_{1,s,w}(\vbx)),\ldots,\sigma(WU^{-1}T_{\frac{n-w}{s},s,w}(\vbx)) \right)
        \\
        &=& g\circ F_{WU^{-1}}(\vbx)
    \end{eqnarray*}
    That is $Uf = g\circ F_{WU^{-1}}$. This implies that $Uf$ and $f=g\circ F_{W}$ are identically distributed: Since $W$ is regular, $W$ and $WU^{-1}$ are identically distributed and hence also $F_{W}$ and $F_{WU^{-1}}$. Since $g$ is independent of $W$, we conclude that $f$ and $Uf$ are identically distributed as well.
\end{proof}

We say that a random function $f:X\to\reals$ is $(q,p)$-balanced if w.p. $\ge q$ over the choice of $f$ we have that $\mu_X(f\ge 0) \ge p$ and $\mu_X(f\le 0) \ge p$. We note that if $f$ is a layered network with an odd activation, and $X=\cc(\vbx_0)$ for some $\vbx_0\in (\reals^d)^n$ then $f$ is $(1,1/2)$-balanced. Indeed, if $A = \{\vbx : f(\bx)\ge 0\}$ then, since $f$ is odd, $A^\complement = -I\cdot A$. Hence,
\[
\mu_X(f\ge 0) = \mu_X(A) \stackrel{\mu_X\text{ is $\so(d)$ invariant}}{=} \mu_X(-I\cdot A)= \mu_X(A^\complement)= \mu_X(f\le 0)
\]
In the second equality from the left we relied on the fact that $d$ is even and hence $-I\in\so(d)$.
The following Theorem shows that a random function that is balanced and $\so(d)$-invariant has adversarial examples w.h.p. 

\begin{theorem}\label{thm:main}
    Let $X$ be transitive $L$-Lipschitz $\so(d)$-metric-space and let $\bx_0\in X$. Let $f:X\to\reals$ be a random function that is $\so(d)$-invariant and $(q,p)$-balanced. Then, w.p. $q(1 - 2e^{-\frac{(d-2)\epsilon^2p^2}{8L^2}})$ either $f(\bx_0)=0$ or there is $\bx$ such that $d(\bx,\bx_0)\le \epsilon$ and $\sign(f(\bx))\ne \sign(f(\bx_0))$
\end{theorem}
Before proving Theorem \ref{thm:main} we note that it implies Theorem \ref{thm:odd}. Indeed, we can take $X=\cc(\vbx_0)$ which is the $\|\vbx_0\|$-Lipschitz (lemma \ref{lem:conv_inp_space_is_lip}), $\epsilon = \frac{\tau\|\vbx_0\|_\spec}{\sqrt{d-2}}$ and $f$ to be a random layered network with odd activations, whose first layer is a regular random convolutional layer that is independent from the remaining layers. Since $f$ is $\so(d)$-invariant and $(1,1/2)$-balanced, Theorem \ref{thm:main} implies that w.p. $1 - 2e^{-\frac{\tau^2}{32}}$ either $f(\bx_0)=0$ or there is $\bx$ such that $d(\bx,\bx_0)\le  \frac{\tau\|\vbx_0\|_\spec}{\sqrt{d-2}}$ and $\sign(f(\bx))\ne \sign(f(\bx_0))$. This implies Theorem \ref{thm:odd}.

Theorem \ref{thm:main} also implies Theorem \ref{thm:relu_nice} via a similar argument. Take again $X=\cc(\vbx_0)$, $\epsilon = \frac{\tau\|\vbx_0\|_\spec}{\sqrt{d-2}}$, and let $f$ to be a random convolutional ReLU random network as in Theorem \ref{thm:relu_nice}. We have that $f$ is $\so(d)$-invariant (lemma \ref{lem:conv_nets_are_inv}). However, as opposed to random network with odd activations it is not immediate that $f$ is balanced. In Lemma \ref{lem:relu_balanced} below we do show that this is the case that $f$ is $(1-o_d(1),1/\log(d))$-balanced. Hence, Theorem \ref{thm:main} implies that w.p. $(1-o_d(1))\left(1 - 2e^{-\Omega(\tau^2/\log^2(d))}\right)$ either $f(\bx_0)=0$ or there is $\bx$ such that $d(\bx,\bx_0)\le  \frac{\tau\|\vbx_0\|_\spec}{\sqrt{d-2}}$ and $\sign(f(\bx))\ne \sign(f(\bx_0))$. This implies Theorem \ref{thm:relu_nice}.

\begin{proof} (of Theorem \ref{thm:main})
    Let $U\in\so(d)$ be a random matrix. We can assume w.l.o.g. that $f=Ug$ where $g$ is distributed as $f$ and independent of $U$. Now, by conditioning on $g$, and since w.p. $\ge q$ we have $\mu_X(g\ge 0) \ge p$ and $\mu_X(g\le 0) \ge p$, the Theorem follows from Lemma \ref{lem:main} below.  
\end{proof}

\begin{lemma}\label{lem:main}
    Let $X$ be transitive $L$-Lipschitz $\so(d)$-metric-space and let $\bx_0\in X$. Let $f:X\to\reals$ be a function such that $\mu_X(f\ge 0) \ge p$ and $\mu_X(f \le 0) \ge p$. Let $U\in\so(d)$ be a random matrix.
    Then, w.p. $1 - 2e^{-\frac{(d-2)\epsilon^2p^2}{8L^2}}$ either $Uf(\bx_0)=0$ or there is $\bx$ such that $d(\bx,\bx_0)\le \epsilon$ and $\sign(Uf(\bx))\ne \sign(Uf(\bx_0))$
\end{lemma}
\begin{proof}
Let $A^+ = \{\bx : f(\bx)>0\}$ and $A^- = \{\bx : f(\bx)<0\}$. By Theorem \ref{thm:iso_for_inv} and the fact that $\mu_X(A^+) \ge p$ and $\mu_X(A^-) \ge p$ we have
\[
\mu_X(A^+_\epsilon\cap A^-_\epsilon) \ge 1 -2 e^{-\frac{(d-2)\epsilon^2p^2}{8L^2}}
\]
Hence, w.p. $1 -2 e^{-\frac{(d-2)\epsilon^2p^2}{8L^2}}$ over the choice of $U$ we have that $U^{-1}\bx_0\in A^+_\epsilon\cap A^-_\epsilon$. It is enough to show that in this case either $Uf(\bx_0)=0$ or there is $\bx$ such that $d(\bx,\bx_0)\le \epsilon$ and $\sign(Uf(\bx))\ne \sign(Uf(\bx_0))$.

Indeed, suppose that $f(U^{-1}\bx_0)> 0$ (the case $f(U^{-1}\bx_0)< 0$ is similar and the case $f(U^{-1}\bx_0)= 0$ is immediate). Since $U^{-1}\bx_0\in  A^-_\epsilon$ there is $\by\in A^-$ such that $d(\by, U^{-1}\bx_0)$. Let $\bx = U\by$. We have
\[
d(\bx,\bx_0) = d(U^{-1}\bx,U^{-1}\bx_0)= d(U^{-1}U\by,U^{-1}\bx_0)= d(\by,U^{-1}\bx_0)\le\epsilon
\]
Likewise, 
\[
Uf(\bx) = f(U^{-1}\bx) =  f(U^{-1}U\by) = f(\by) \le 0
\]
Hence, $\sign(Uf(\bx))\ne \sign(Uf(\bx_0))$
\end{proof}

\section{Balance-ness of random ReLU networks}
In this section we will prove the following Lemma
\begin{lemma}\label{lem:relu_balanced}
    Fix a constant $l$.
    Let $f:(\reals^d)^n\to\reals$ be a random layered network with $l$ independent convolutional Xavier layers with ReLU activation, except the last layer which has linear activation.
    Assume that the number of channels in each layer, as well as $d$ are all $\omega(\log(nd))$. Fix $\vbx_0\in \left(R\cdot\sphere^{d-1}\right)^n$. Then, $f|_{\cc(\vbx_0)}$ is $(1-o_d(1),1/\log(d))$-balanced. 
\end{lemma}
\subsection{Proof of Lemma \ref{lem:relu_balanced}}
Since the ReLU is homogeneous, we can assume w.l.o.g. that $R=\sqrt{d}$.
Let $\vbx^1,\ldots,\vbx^m$ be i.i.d. uniform points in $X=\cc(\bx_0)$ with $m=\lfloor\sqrt{\log(d)}\rfloor$. We say that $f$ {\em separates} $\vbx^1,\ldots,\vbx^m$ if there are $1\le i,j\le m$ such that $f(\vbx^i)<0<f(\vbx^j)$. Let $A_m$ be the event that $f$ separates $\vbx^1,\ldots,\vbx^{m}$. We will show that 
\begin{lemma}\label{lem:core_of_relu_balanced}
$\Pr_{\vbx^1,\ldots,\vbx^m,f}(A_m) = 1-o_d(1)$.    
\end{lemma}
Before proving Lemma \ref{lem:core_of_relu_balanced}, we will explain how it implies Lemma \ref{lem:relu_balanced}. . We have
\begin{eqnarray*}
    \Pr_{\vbx^1,\ldots,\vbx^m,f}(A_m) &= & \Pr(f\text{ is $\beta$-balanced})\E_{f}\left[\Pr_{\vbx^1,\ldots,\vbx^m}(A_m)|f\text{ is $\beta$-balanced}\right]
    \\
    && + \Pr(f\text{ is not $\beta$-balanced})\E_{f}\left[\Pr_{\vbx^1,\ldots,\vbx^m}(A_m)|f\text{ is not $\beta$-balanced}\right]
    \\
    &\le & \Pr(f\text{ is $\beta$-balanced}) + \E_{f}\left[\Pr_{\vbx^1,\ldots,\vbx^m}(A_m)|f\text{ is not $\beta$-balanced}\right]
    \\
    &\le & \Pr(f\text{ is $\beta$-balanced}) + 1-(1-\beta)^m
\end{eqnarray*}
Taking $\beta = 1/m^2$ we conclude Lemma \ref{lem:relu_balanced} as
\[
\Pr(f\text{ is $1/\log(d)$-balanced}) \ge \Pr(f\text{ is $1/m^2$-balanced})\ge \Pr_{\vbx^1,\ldots,\vbx^m,f}(A_m) - 1  +  (1-1/m^2)^m = 1-o_d(1)
\]
Hence, it is enough to prove Lemma \ref{lem:core_of_relu_balanced}.
We have that $f=F_{W_{l}}\circ\ldots\circ F_{W_1}$ where $F_{W_1},\ldots,F_{W_{l}}$ are independent Xavier convolutional layers with ReLU activation in all layers, except the last one ($F_{W_{l}}$) whose activation is the identity function. Assume that $F_{W_v}:\left(\reals^{d_{v-1}}\right)^{n_{v-1}}\to \left(\reals^{d_{v}}\right)^{n_{v}}$ is a convolutional layer of width $w_v$, stride $s_v$ and $d_v$ channels. Note that since the range of $f$ is $\reals$ we have $n_l=d_l=1$.
Denote $\Psi = \sqrt{\frac{2^{l-1}}{n_{l-1}d_{l-1}}}F_{W_{l-1}}\circ\ldots\circ F_{W_1}$.
We will prove Lemma \ref{lem:core_of_relu_balanced} in three steps corresponding to the following three Lemmas
\begin{lemma}\label{lem:core_1}
    W.p. $1-o_d(1)$, for any $1\le i<j\le m$ and $1\le t\le n$, $\inner{\bx^i_t,\bx^j_t} \le d/2$.
\end{lemma}

\begin{lemma}\label{lem:core_2}
    Given that for any $1\le i<j\le m$ and $1\le t\le n$, $\inner{\bx^i_t,\bx^j_t} \le d/2$. We have w.p. $1-o_d(1)$ that for any $1\le i<j\le m$ it holds that $\|\Psi(\vbx^i)-\Psi(\vbx^j)\|\ge \beta$ and $\|\Psi(\vbx^i)\|\le 2$, where $\beta>0$ is a constant depending only on the depth $l$
\end{lemma}

\begin{lemma}\label{lem:core_3}
    Fix a constant $\beta>0$. Given that for any $1\le i<j\le m$ it holds that $\|\Psi(\vbx^i)-\Psi(\vbx^j)\|\ge \beta$ and $\|\Psi(\vbx^i)\|\le 2$ we have w.p. $1-o_d(1)$ that $f$ separates $\vbx^1,\ldots\vbx^m$ 
\end{lemma}
Lemmas \ref{lem:core_1}, \ref{lem:core_2} and \ref{lem:core_3} clearly implies Lemma \ref{lem:core_of_relu_balanced}, so it remains to prove them. Lemma \ref{lem:core_1} is a simple consequence of the fact that if $\bx,\by\in \sqrt{d}\sphere^{d-1}$ are uniform and independent then $\Pr(\inner{\bx,\by} \ge t)\le e^{-\frac{t^2}{2d}}$ (e.g. \cite{vershynin2018high} chapter 3). Hence, the probability that for some $1\le i<j\le m$ and $1\le t\le n$, $\inner{\bx^i_t,\bx^j_t} > d/2$ is at most $2\binom{m}{2}ne^{-\frac{d}{8}} = o_d(1)$ where the last inequity is correct as we assume that $d=\omega(\log(n))$.
It remains to prove Lemmas \ref{lem:core_2} and \ref{lem:core_3}, which we do next

\subsubsection{Proof of Lemma \ref{lem:core_2}}
In order to prove Lemma \ref{lem:core_2} we will use a result of \citet{daniely2016toward}, which shows that w.h.p. $\inner{\Psi(\vbx),\Psi(\vby)}\approx k(\vbx,\vby)$ where $k:\left(\sqrt{d}\cdot\sphere^{d-1}\right)^n\times \left(\sqrt{d}\cdot\sphere^{d-1}\right)^n\to\reals$ is a kernel function with an explicit formula, which we define next. Let
\[
\hat{\sigma}(u) = \frac{1}{\pi}\left(u(\pi-\arccos(u))+\sqrt{1-u^2}\right)
\]
We note that $\hat{\sigma}$ is non-negative and monotone in $[-1,1]$, and satisfies $\hat\sigma(1)=1$.
For $\vbx,\vby\in \left(\sqrt{d}\cdot\sphere^{d-1}\right)^n$, $0\le v\le l-1$ and $1\le t\le n_i$ we define recursively
\[
k_{0,t}(\vbx,\vby) = \frac{\inner{\bx_t,\by_t}}{d}\text{ and }k_{v,t}(\vbx,\vby) = \hat{\sigma}\left(\frac{1}{w_v}\sum_{r=1}^{w_v}k_{v-1,s_v(t-1)+r}(\vbx,\vby)\right)
\]
Finally, let $k(\vbx,\vby) = k_{l,1}(\vbx,\vby)$. The following Lemma shows that w.p. $1-o_d(1)$, for any $1\le i,j\le m$, $\inner{\Psi(\vbx^i),\Psi(\vbx^i)}=k(\vbx^i,\vbx^j)+o_d(1)$

\begin{lemma}\label{lem:emp_kernel}\cite{daniely2016toward}
    Suppose that for any $1\le v\le l-1$ we have $d_i \gtrsim \frac{l^2\log(ln/\delta)}{\epsilon^2}$ and that $\epsilon \lesssim \frac{1}{l}$. Then,
    \[
    \Pr\left(\left|k(\vbx,\vby) - \inner{\Psi(\vbx),\Psi(\vby)}\right| > \epsilon\right) <\delta
    \]
\end{lemma}
Next, we show by induction on $v$ that if for any $1\le i<j\le m$ and $1\le t\le n$, $\inner{\bx^i_t,\bx^j_t} \le d/2$ then 
$k_{v,t}(\vbx^i,\vbx^j)\le  \hat{\sigma}^{\circ v}(1/2) < 1$ for $i\ne j$ and $k_{v,t}(\vbx^i,\vbx^i)=1$. Indeed,
\begin{eqnarray*}
    k_{v,t}(\vbx^i,\vbx^j) &=& \hat{\sigma}\left(\frac{1}{w_v}\sum_{r=1}^{w_v}k_{v-1,s_v(t-1)+r}(\vbx^i,\vbx^j)\right)
    \\
    &\stackrel{\hat{\sigma}\text{ monotonicity and induction hypothesis}}{\le} & \hat{\sigma}(\hat{\sigma}^{\circ (v-1)}(1/2))
    \\
    &=& \hat{\sigma}^{\circ v}(1/2)
\end{eqnarray*}
and
\[
k_{v,t}(\vbx^i,\vbx^i) = \hat{\sigma}\left(\frac{1}{w_v}\sum_{r=1}^{w_v}k_{v-1,s_v(t-1)+r}(\vbx^i,\vbx^i)\right)
    \stackrel{\text{ induction hypothesis}}{=}  \hat{\sigma}(1)=1
\]
Hence, we have w.p. $1-o_d(1)$ that $\|\Psi(\vbx_i)\|^2 = 1 + o_d(1)$ and that
\begin{eqnarray*}
    \|\Psi(\vbx_i)-\Psi(\vbx_j)\|^2 = 1 - 2k(\vbx^i,\vbx^j) + 1 + o_d(1) \ge 2(1-\hat{\sigma}^{\circ (l-1)}(1/2)) + o_d(1)
\end{eqnarray*}
which proves Lemma \ref{lem:core_2}

\subsubsection{Proof of Lemma \ref{lem:core_3}}

\begin{theorem}[Sudakov e.g. \cite{van2014probability} chapter 6.1]\label{lem:sudakov}
    Let $\bx_1,\ldots,\bx_m\in \reals^d$ be vectors such that $\|\bx_i-\bx_j\|\ge\alpha$ for any $1\le i<j\le m$. Let $\bw\in\reals^d$ be standard Gaussian. Then, $\E \max_i \inner{\bw,\bx_i} \gtrsim \alpha\sqrt{\log(m)}$ 
\end{theorem}
Let $Z = \max_{1\le i\le m}f(\vbx^i)$. Given that for any $1\le i<j\le m$ it holds that $\|\Psi(\vbx^i)-\Psi(\vbx^j)\|\ge \beta$ and $\|\Psi(\vbx^i)\|\le 2$ we have that $\E Z = \omega(1)$ by Sudakov Lemma. Since $W_l\mapsto \max_{i}\inner{W_l,\sqrt{d_{l-1}n_{l-1}}\Psi(\bx_i)}$ is $2$-Lipchitz, and $W_l$ is a matrix with i.i.d. centered Gaussians with variance $\frac{1}{d_{l-1}n_{l-1}}$,
Gaussian concentration (e.g. \cite{vershynin2018high} section 5.2.1) implies that $\var(Z) \le 4$. Hence, w.p. $1-o_d(1)$, $Z>0$, implying that there is $i$ such that $f(\vbx^i)>0$. Similarly, w.p. $1-o_d(1)$ there is also $i$ such that $f(\vbx^i)<0$. This proves Lemma \ref{lem:core_3}

\section*{Acknowledgments}
The research described in this paper was funded by the European Research Council (ERC) under the European Union’s Horizon 2022 research and innovation program (grant agreement No. 101041711), the Israel Science Foundation (grant number 2258/19), and the Simons Foundation (as part of the Collaboration on the Mathematical and Scientific Foundations of Deep Learning).

\bibliography{bib}

\end{document}